\def\year{2017}\relax
\documentclass[letterpaper]{article}
\usepackage{aaai17}
\usepackage{times}
\usepackage{helvet}
\usepackage{courier}

\usepackage{graphicx}
\usepackage{algorithm, algpseudocode}
\usepackage{tensor}
\usepackage{amsmath}
\usepackage{amsthm}
\usepackage{relsize}
\usepackage{nicefrac}
\usepackage{url}
\usepackage[hang,flushmargin]{footmisc} 

\usepackage[round]{natbib}
\bibpunct{(}{)}{,}{a}{}{;}
\newtheorem{lem}{Lemma}
\newtheorem{theorem}[lem]{Theorem}
\newtheorem{prop}[lem]{Proposition}

\algtext*{EndIf}
\algtext*{EndWhile}
\algtext*{EndFor}
\algtext*{EndProcedure}

\newcommand{\citenoun}[1]{\citet{#1}}
\newcommand{\citenop}[1]{\citealt{#1}}

\DeclareMathOperator*{\expect}{{\mathlarger {\mathbf E}}}

\newcommand{\expects}{\expect\nolimits}

\algnewcommand{\LineComment}[1]{\Statex \(\triangleright\) {\em #1}}
\algrenewcommand\algorithmicindent{0.75em}%

\begin{document}
%
\title{Self-Correcting Models for Model-Based Reinforcement Learning}
\author{Erik Talvitie\\
Department of Mathematics and Computer Science\\
Franklin \& Marshall College\\
Lancaster, PA 17604-3003\\
\texttt{erik.talvitie@fandm.edu}
}
\maketitle


\begin{abstract}
  When an agent cannot represent a perfectly accurate model of its
  environment's dynamics, model-based reinforcement learning (MBRL)
  can fail catastrophically. Planning involves composing the
  predictions of the model; when flawed predictions are composed, even
  minor errors can compound and render the model useless for
  planning. Hallucinated Replay \citep{talvitie2014model} trains the
  model to ``correct'' itself when it produces errors, substantially
  improving MBRL with flawed models. This paper theoretically analyzes
  this approach, illuminates settings in which it is likely to be
  effective or ineffective, and presents a novel error bound, showing
  that a model's ability to self-correct is more tightly related to
  MBRL performance than one-step prediction error. These
  results inspire an MBRL algorithm for deterministic MDPs with
  performance guarantees that are robust to model class limitations.
\end{abstract}

\section{Introduction} \label{sec:intro}

In model-based reinforcement learning (MBRL) the agent learns a
predictive model of its environment and uses it to make decisions. The
overall MBRL approach is intuitively appealing and there are many
anticipated benefits to learning a model, most notably sample
efficiency \citep{szita2010model}. Despite this, with few exceptions
(e.g. \citenop{Abbeel:2006:UIM:1143844.1143845}), model-free methods
have been far more successful in large-scale problems. Even as
model-learning methods demonstrate increasing prediction accuracy in
high-dimensional domains (e.g. \citenop{bellemare2014skip},
\citenop{oh2015action}) this rarely corresponds to improvements in
control performance.

One key reason for this disparity is that model-free methods are
generally robust to representational limitations that prevent
convergence to optimal behavior. In contrast, when the model
representation is insufficient to perfectly capture the environment's
dynamics (even in seemingly innocuous ways), or when the planner
produces suboptimal plans, MBRL methods can fail catastrophically. If
the benefits of MBRL are to be gained in large-scale problems, it is
vital to understand how MBRL can be effective even when the model and
planner are fundamentally flawed.

Recently there has been growing awareness that the standard measure of
model quality, one-step prediction accuracy, is an inadequate proxy
for MBRL performance. For instance, \citenoun{sorg2010reward} and
\citenoun{joseph2013reinforcement} both pointed out that the most
accurate model by this measure is not necessarily the best for
planning. Both proposed optimizing model parameters for control
performance using policy gradient methods. Though appealing in its
directness, this approach arguably discards some of the benefits of
learning a model in the first place.

\begin{figure}
\centering
\includegraphics[width=.95\linewidth]{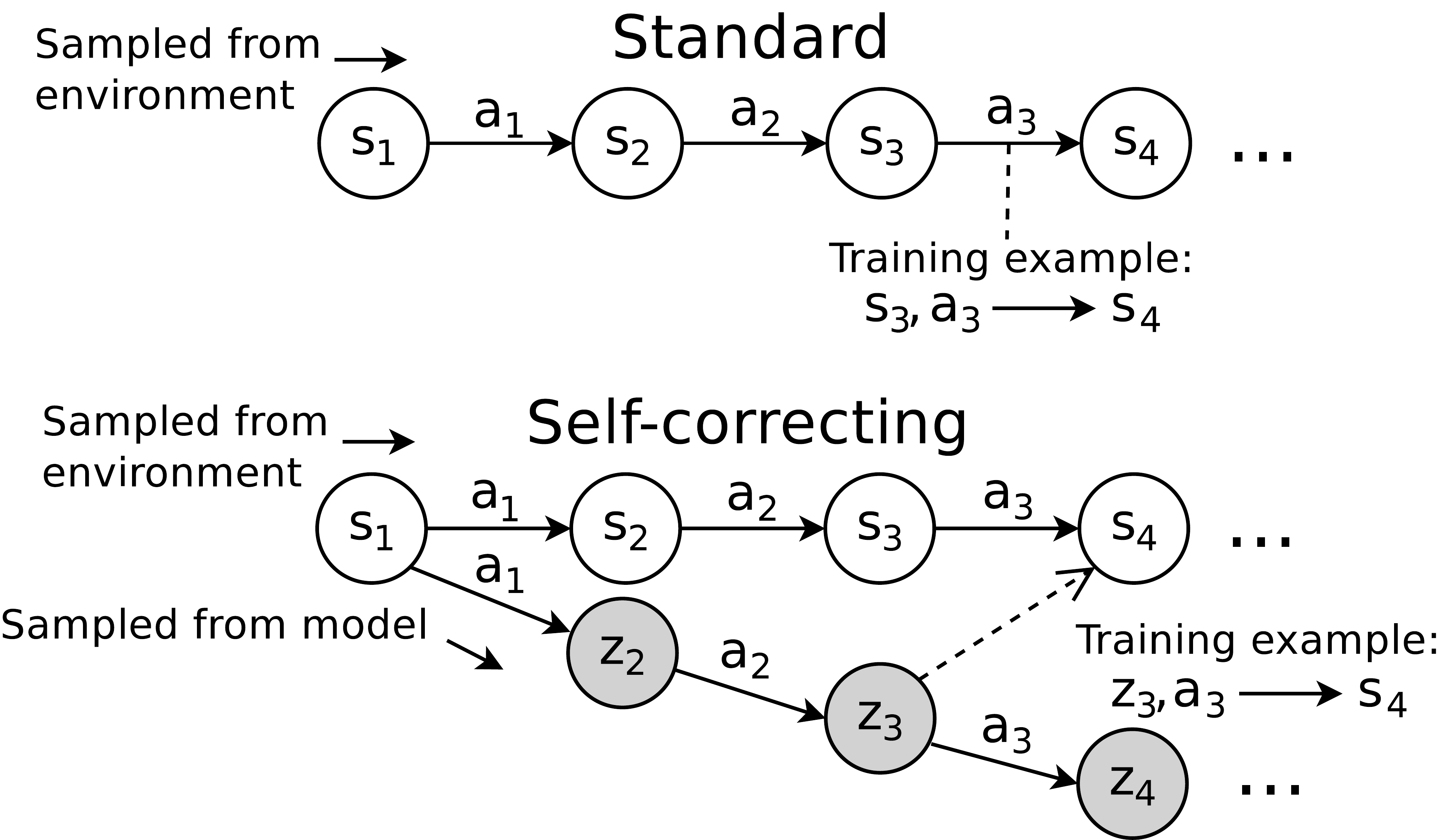}
\caption{Top: training models to predict environment states from
  environment states. Bottom: training models to predict environment
  states from states sampled from the model.}
\label{fig:trainingex}
\end{figure}

\citenoun{talvitie2014model} pointed out that one-step prediction
accuracy does not account for how the model behaves when composed with
itself and introduced the Hallucinated Replay
meta-algorithm to address this. As illustrated in Figure
\ref{fig:trainingex}, this approach rolls out the model and
environment in parallel, training the model to predict the correct
environment state ($s_4$) even when its input is an incorrect sampled
state ($z_3$). This effectively causes the model to ``self-correct''
its rollouts. Hallucinated Replay was shown to enable meaningful
planning with flawed models in examples where the standard approach
failed. However, it offers no theoretical
guarantees. \citenoun{venkatraman2015improving} and
\citenoun{oh2015action} used similar approaches to improve
models' long-range predictions, though not in the MBRL setting.

This paper presents novel error bounds that reveal the theoretical
principles that underlie the empirical success of Hallucinated
Replay. It presents negative results that identify settings where
hallucinated training would be ineffective (Section
\ref{sec:hallucinatedGeneral}) and identifies a case where it yields a
tighter performance bound than standard training
(Section \ref{sec:hallucinatedDetBlind}). This result allows the derivation of a
novel MBRL algorithm with theoretical performance guarantees that are
robust to model class limitations (Section \ref{sec:hdaggermc}). The analysis
also highlights a previously underexplored practical concern with this
approach, which is examined empirically (Section \ref{sec:unrolled}).

\subsection{Notation and background}

We focus on {\em Markov decision processes} (MDP). The environment's
initial state $s_1$ is drawn from a distribution $\mu$. At each step
$t$ the environment is in a state $s_t$. The agent selects an action
$a_t$ which causes the environment to transition to a new state
sampled from the transition distribution:
$s_{t+1} \sim P_{s_t}^{a_t}$.  The environment also emits a reward,
$R(s_t, a_t)$. For simplicity, assume that the reward function is
known and is bounded within $[0, M]$.

A {\em policy} $\pi$ specifies a way to behave in the MDP. Let
$\pi(a \mid s) = \pi_s(a)$ be the probability that $\pi$ chooses
action $a$ in state $s$. For a sequence of actions $a_{1:t}$ let
$P(s' \mid s, a_{1:t}) = P_s^{a_{1:t}}(s')$ be the probability of
reaching $s'$ by starting in $s$ and taking the actions in the
sequence. For any state $s$, action $a$, and
policy $\pi$, let $D^t_{s, a, \pi}$ be the state-action distribution
obtained after $t$ steps, starting with state $s$ and action $a$ and
thereafter following policy $\pi$. For a state action distribution
$\xi$, let
$D^t_{\xi, \pi} = \expects_{(s, a) \sim \xi} D^t_{s, a, \pi}$. For a
state distribution $\mu$ let
$D^t_{\mu, \pi} = \expects_{s \sim \mu, a \sim \pi_s} D^t_{s, a,
  \pi}$.
For some discount factor $\gamma \in [0, 1)$, let
$D_{\mu, \pi} = (1 - \gamma) \sum_{t=1}^{\infty}\gamma^{t-1}D^t_{\mu,
  \pi}$
be the infinite-horizon {\em discounted} state-action distribution
under policy $\pi$.

The $T$-step {\em state-action value} of a policy, $Q^\pi_T(s, a)$
represents the expected discounted sum of rewards obtained by taking
action $a$ in state $s$ and executing $\pi$ for an additional $T-1$
steps:
$Q^\pi_T(s, a) = \sum_{t = 1}^{T}\gamma^{t-1} \expects_{(s', a') \sim
  D^t_{s, a, \pi}} R(s', a')$.
Let the $T$-step {\em state value}
$V^\pi_T(s) = \expects_{a \sim \pi_s}[Q^\pi_T(s, a)]$. For infinite
horizons we write $Q^\pi = Q^\pi_{\infty}$, and
$V^\pi = V^\pi_{\infty}$. The agent's goal will be to learn a policy
$\pi$ that maximizes $\expects_{s \sim \mu}[V^\pi(s)]$.

The MBRL approach is to learn a model $\hat{P}$, approximating $P$,
and then to use the model to produce a policy via a planning
algorithm. We let $\hat{D}$, $\hat{Q}$, and $\hat{V}$ represent the
corresponding quantities using the learned model. Let $\mathcal{C}$
represent the {\em model class}, the set of models the learning
algorithm could possibly produce. Critically, in this paper, it is not
assumed that $\mathcal{C}$ contains a perfectly accurate model.

\section{Bounding value error}

We consider an MBRL architecture that uses the simple one-ply Monte
Carlo planning algorithm (one-ply MC), which has its roots in the
``rollout algorithm'' \citep{tesauro1996line}. For every state-action
pair $(s, a)$, the planner executes $N$ $T$-step ``rollouts'' in
$\hat{P}$, starting at $s$, taking action $a$, and then following a
{\em rollout policy} $\rho$. Let $\bar{Q}(s, a)$ be the average
discounted return of the rollouts. For large $N$, $\bar{Q}$ will
closely approximate $\hat{Q}_T^{\rho}$ \citep{kakade2003sample}. The agent will select its
actions greedily with respect to
$\bar{Q}$. \citenoun{talvitie2015agnostic} bounds the performance of
one-ply MC in terms of model quality.

For a policy $\pi$ and state-action distribution $\xi$, let
$\epsilon_{val}^{\xi, \pi, T} = \expect_{(s, a) \sim \xi}\big[|Q_T^{\pi}(s,
a) - \hat{Q}_T^{\pi}(s, a)|\big]$
be the error in the $T$-step state-action values the model assigns to
the policy under the given distribution. Then the following result can
be straightforwardly adapted from one provided by
\citenoun{talvitie2015agnostic}.

\begin{lem} \label{lem:mcvaluebound}
  Let $\bar{Q}$ be the state-action value function returned by
  applying one-ply Monte Carlo to the model $\hat{P}$ with rollout
  policy $\rho$ and rollout depth $T$. Let $\hat{\pi}$ be greedy
  w.r.t. $\bar{Q}$. For any policy $\pi$ and state-distribution $\mu$,
  \begin{align*}
    \expect_{s \sim \mu}\big[V^\pi(s) - V^{\hat{\pi}}(s)\big] \le
    \frac{4}{1 - \gamma}\epsilon_{val}^{\xi, \rho, T}  +\epsilon_{mc}, 
    \end{align*}
    where we let $\xi(s, a) = \frac{1}{2} D_{\mu, \hat{\pi}}(s, a) + \frac{1}{4}
  D_{\mu, \pi}(s, a) + \frac{1}{4}\left((1 - \gamma) \mu(s)
    \hat{\pi}_s(a) + \gamma \sum_{z, b} D_{\mu, \pi}(z, b)
    P_{z}^{b}(s) \hat{\pi}_s(a) \right)$ and  $\epsilon_{mc} = \frac{4}{1 - \gamma}\|\bar{Q} -
  \hat{Q}^{\rho}_T\|_\infty + \frac{2}{1 - \gamma} \|BV^{\rho}_T -
  V^{\rho}_T\|_{\infty}$ (here $B$ is the Bellman operator).
\end{lem} 
The $\epsilon_{mc}$ term represents error due to limitations of the
planning algorithm: error due to the sample average $\bar{Q}$ and the
sub-optimality of the $T$-step value function with respect to $\rho$. The
$\epsilon_{val}^{\xi, \rho, T}$ term represents error due to the
model parameters. The key factor in the model's usefulness for
planning is the accuracy of the value it assigns to the rollout policy
in state-actions visited by $\pi$ and $\hat{\pi}$. Our goal in the
next sections is to bound $\epsilon_{val}^{\xi, \rho, T}$ in terms
of measures of model accuracy, ultimately deriving insight into how to
train models that will be effective for MBRL. Proofs may be found in
the appendix.

\subsection{One-step prediction error} \label{sec:onestep}

Intuitively, the value of a policy should be accurate if the model is
accurate in states that the policy would visit. We can adapt a bound
from \citenoun{ross2012agnostic}. 
\begin{lem} \label{lem:onestepbound} 
For any policy $\pi$ and
  state-action distribution $\xi$,
\begin{align*}
\epsilon_{val}^{\xi, \pi, T} \le \frac{M}{1 - \gamma}\sum_{t = 1}^{T-1} (\gamma^{t} - 
\gamma^{T})\expect_{(s, a) \sim
  D^t_{\xi, \pi}}\big[\|P_s^a - \hat{P}_s^a\|_1\big].
\end{align*}
\end{lem} 

Combining Lemmas \ref{lem:mcvaluebound} and \ref{lem:onestepbound}
yields an overall bound on control performance in terms of the model's
prediction error. This result matches common MBRL practice; it
recommends minimizing the model's one-step prediction error. It
acknowledges that the model may be imperfect by allowing it to have
one-step error in unimportant (i.e. unvisited) states. However, if
limitations of the model class prevent the model from achieving low
error in important states, this bound can be quite loose, as the
following example illustrates.

\begin{figure}
\centering
\includegraphics[width=0.95\linewidth]{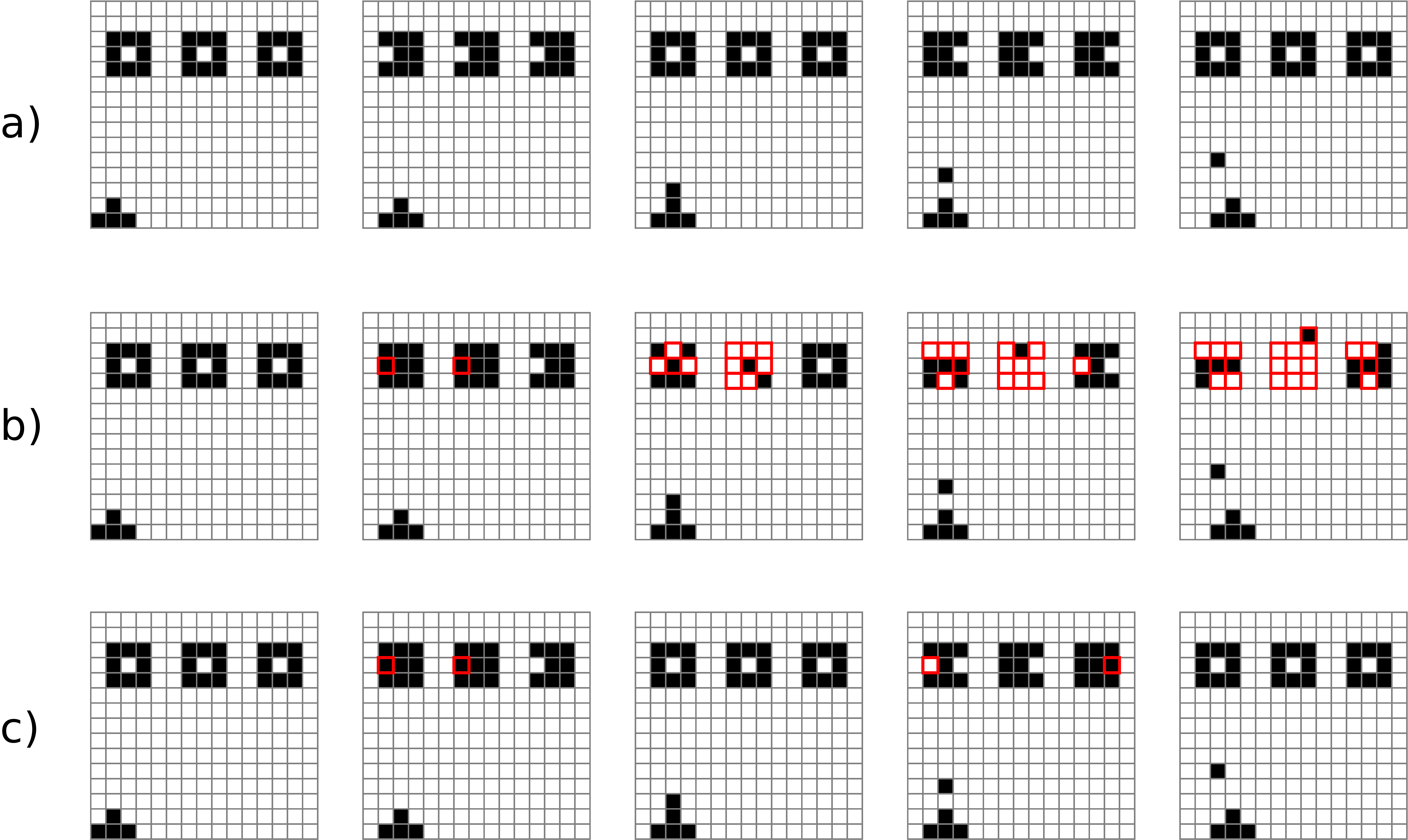}
\caption{The Shooter game. a) Example of the real dynamics. b)
  Propagating errors (red outlines) in a model
  optimized for one-step error. c) A model optimized for
  multi-step error.}
\label{fig:shooter}
\end{figure}

Consider the ``Shooter'' domain introduced by
\citenoun{talvitie2015agnostic}, pictured in Figure
\ref{fig:shooter}a. The agent moves a spaceship left and right at the
bottom of the screen. It can fire bullets upward, but each one has a
cost (-1 reward). If a bullet hits one of the three targets, the
agent receives 10 reward. Each target has a ``bullseye'' (the white
dots). If a bullet hits the same column as a bullseye, the agent
receives an additional 10 reward. Though the control problem is
simple, the state/observation space is high-dimensional due to the
many possible configurations of objects on the screen.

In the original Shooter the bullseyes remained still but here they
move back and forth across the targets. As such, the problem is
second-order Markov; when the bullseye is in the center, one cannot
predict its next position without knowing its previous position. The
agent, however, will use a factored Markov model, predicting each
pixel conditioned on the current image. It cannot accurately predict
the bullseyes' movement, though it can predict everything else
perfectly.

One might imagine that this limitation would be fairly minor; the
agent can still obtain reward even if it cannot reliably hit the
bullseyes. However, consider the sample rollout pictured in Figure
\ref{fig:shooter}b. Here each image is sampled from a model's one-step
predictions, and is then given as input for the next predictions. This
model has the lowest possible one-step prediction error. Still, as
anticipated, it does not correctly predict the movement of the
bullseyes in the second image. Because of the resulting errors, the
sampled image is unlike any the environment would generate, and
therefore unlike any the model has trained on. The model's uninformed
predictions based on this unfamiliar image cause more errors in the
third image, and so on. Ultimately this model assigns low probability
to a target persisting more than a few steps, making it
essentially useless for planning.

Note, however, that there {\em are} models within this model class
that are useful for planning. Consider the sample rollout pictured in
Figure \ref{fig:shooter}c. The model that generated this rollout makes
the same one-step errors as the previous model when given an
environment state. However, when it encounters an unreasonable sampled
state it still makes reasonable predictions, effectively
``self-correcting.''  \citenoun{talvitie2014model} presents several
similar examples involving various model deficiencies. These examples
illustrate the inadequacy of Lemma \ref{lem:onestepbound} when the
model class is limited. Models with similar one-step prediction error
can vary wildly in their usefulness for planning. The true
distinguisher is the accuracy of predictions far into the future.

\subsection{Multi-step prediction error}

Since $Q^\pi_T(s, a) = \sum_{t = 1}^{T}\gamma^{t-1} \expects_{(s', a') \sim
  D^t_{s, a, \pi}} R(s', a')$, it is straightforward to bound
$\epsilon_{val}^{\xi, \pi, T}$ in terms of {\em multi-step} error.
\begin{lem} \label{lem:multistep}
  For any policy $\pi$ and state-action distribution $\xi$,
\begin{align*}
\epsilon_{val}^{\xi, \pi, T} \le M\sum_{t = 1}^{T} \gamma^{t-1}
                             \expect_{(s, a) \sim \xi}\big[\|D^t_{s, a,
                             \pi} - \hat{D}^t_{s, a, \pi}\|_1\big].
\end{align*}
\end{lem} 

The bound in Lemma \ref{lem:onestepbound} has dependence on
$\frac{1}{1 - \gamma}$ because it effectively assumes the worst
possible loss in value if the model samples an ``incorrect'' state. In
contrast, Lemma \ref{lem:multistep} accounts for the model's ability
to recover after an error, only penalizing it for individual incorrect
transitions. Unfortunately, it is difficult to directly optimize for
multi-step prediction accuracy. Nevertheless, this bound suggests that
algorithms that account for a model's multi-step error will yield more
robust MBRL performance.

\subsection{Hallucinated one-step prediction error}\label{sec:hallucinatedGeneral}

We now seek to formally analyze the practice of hallucinated
training, described in Section
\ref{sec:intro}. \citenoun{venkatraman2015improving} provide some
analysis but in the uncontrolled time series prediction setting. Here
we focus on its impact on control performance in MBRL. As a first
step, we derive a bound based on a model's ability to predict the next
environment state, given a state sampled from the model's own
predictions, i.e. to self-correct. For a policy $\pi$ and state-action
distribution $\xi$ let $J^t_{\xi, \pi}$ represent the {\em joint}
distribution over environment and model state-action pairs if $\pi$ is
executed in both simultaneously. Specifically, let
\begin{align*}
J^t_{\xi, \pi}(s, a, z, b) = \expects_{(s', a') \sim \xi} [D^{t}_{s',
  a', \pi}(s, a) \hat{D}^{t}_{s', a', \pi}(z, b)].
\end{align*}
\begin{lem} \label{lem:hallucinated}
  For any policy $\pi$ and state-action distribution $\xi$,
\begin{align*}
\epsilon_{val}^{\xi, \pi, T}\le
  M \sum_{t = 1}^{T-1} \gamma^t \expect_{(s, a, z, b) \sim J^{t}_{\xi, \pi}}
  \big[\|P_s^a - \hat{P}_z^b\|_1\big].
\end{align*}
\end{lem} 
Inspired by ``Hallucinated Replay'' \citep{talvitie2014model}, we call
the quantity on the right the {\em hallucinated one-step
  error}. Hallucinated one-step error is intended as a proxy for
multi-step error, but having formalized it we may now see that in some
cases it is a poor proxy. Note that, regardless of the policy, the
multi-step and one-step error of a perfect model is 0. This is not
always so for hallucinated error.
\begin{prop} \label{prop:err}
  The hallucinated one-step error of a perfect model may be non-zero.
\end{prop}
\begin{proof}
  Consider a simple MDP with three states $\{s_0, s_h, s_t\}$
  and a single action $a$. In the initial state $s_0$, a fair coin is
  flipped, transitioning to $s_h$ or $s_t$ with equal
  probability, where it stays forever. Consider a perfect
  model $\hat{P} = P$. Then
  $J^1_{s_0, a}(s_h, a, s_t, a) = P_{s_0}^a(s_h)P_{s_0}^a(s_t) = 0.25$.
  However, $|P_{s_h}^a(s_h) - P_{s_t}^a(s_h)| = 1 - 0 =
  1$. Thus, the hallucinated one-step error of a perfect model is
  non-zero.
\end{proof}
Here the environment samples heads and the model samples tails. Given
its own state, the model rightly predicts tails, but incurs error
nevertheless since the environment's next state is heads. Because the
model and environment dynamics are uncoupled, one cannot distinguish
between model error and legitimately different stochastic outcomes. As
such, the hallucinated error is misleading when the true dynamics are
stochastic. This corroborates the conjecture that Hallucinated Replay
may be problematic in stochastic environments
\citep{talvitie2014model}. Note that this observation applies not just
to hallucinated training, but to {\em any} method that attempts to
improve multi-step predictions by comparing sample rollouts from the
model and the environment.

While it may seem limiting to restrict our attention to deterministic
environments, this is still a large, rich class of problems. For
instance, \citenoun{oh2015action} learned models of Atari 2600 games,
which are fully deterministic \citep{hausknecht2014neuroevolution};
human players often perceive them as stochastic due to their
complexity. Similarly, in synthetic RL domains stochasticity is often
added to simulate complex, deterministic phenomena (e.g. robot wheels
slipping on debris), not necessarily to capture inherently stochastic
effects in the world. As in these examples, we shall assume that the
environment is deterministic but complex, so a limited agent will
learn an imperfect, stochastic model.

That said, even specialized to deterministic environments, the bound in
Lemma \ref{lem:hallucinated} is loose for arbitrary policies.
\begin{prop} \label{prop:deterr}
  The hallucinated one-step error of a perfect model may be non-zero,
  even in a deterministic MDP.
\end{prop}
\begin{proof}
  Alter the coin MDP, giving the agent two actions which fully
  determine the coin's orientation. The original dynamics can be
  recovered via a stochastic policy that randomly selects $s_h$ or
  $s_t$ and then leaves the coin alone.
\end{proof}
\citenoun{oh2015action} tied action selection to the environment state
only (rather than independently selecting actions in the environment
and model). This prevents stochastic decoupling but may fail to train the
model on state-action pairs that the policy would reach under the
model's dynamics. 

\subsection{A Tighter Bound}\label{sec:hallucinatedDetBlind}

In the remainder of the paper we assume that the environment is
deterministic. Let $\sigma_{s}^{a_{1:t}}$ be the unique state that
results from starting in state $s$ and taking the action sequence
$a_{1:t}$. The agent's model will still be stochastic. 

Recall that our goal is to bound the value error under the one-ply MC
rollout policy. Proposition \ref{prop:deterr} shows that hallucinated
error gives a loose bound under arbitrary policies. We now focus on
{\em blind policies} \citep{06icml-psr-exploration}. A blind policy
depends only on the action history, i.e.
$\pi(a_t \mid s_t, a_{1:t-1}) = \pi(a_t \mid a_{1:t-1})$. This class
of policies ranges from stateless policies to open-loop action
sequences. It includes the uniform random policy, a common rollout
policy.

For any blind policy $\pi$ and
state-action distribution $\xi$, let $H^t_{\xi, \pi}$ be the
distribution over environment state, model state, and action if
a single action sequence is sampled from $\pi$ and then executed in
both the model and the environment. So,
\begin{align*}
&H^1_{\xi, \pi}(s_1, z_1, a_1) =\xi(s_1, a_1) \text{ when } z_1 = s_1
  \text{ (0
otherwise)};\\
&H^2_{\xi, \pi}(s_2, z_2, a_2) = \expects_{(s_1, a_1) \sim
  \xi}[\pi(a_2 \mid a_1) P_{s_1}^{a_1}(s_2)\hat{P}_{s_1}^{a_1}(z_s)];\\
&\text{and for $t > 2$, } H^t_{\xi, \pi}(s_t, z_t, a_t) =\\
&\hspace{0.05in} \expects_{(s_1, a_1) \sim \xi}
\Big[\textstyle\sum_{a_{2:t-1}} \pi(a_{2:t} \mid a_1) P_{s_1}^{a_{1:t-1}}(s_t)
\hat{P}_{s_1}^{a_{1:t-1}}(z_t)\Big].
\end{align*}
\begin{lem} \label{lem:detblind}
  If $P$ is deterministic, then for any blind policy $\pi$ and any
  state-action distribution $\xi$,
\begin{align*}
\epsilon_{val}^{\xi, \pi, T} \le 2M \sum_{t = 1}^{T-1}\gamma^t
  \expect_{(s, z, a) \sim H^{t}_{\xi, \pi}} \big[1 - \hat{P}_z^a(\sigma^a_s)\big].
\end{align*}
\end{lem}
We can also show that, in the deterministic setting, Lemma
\ref{lem:detblind} gives an upper bound for multi-step error (Lemma
\ref{lem:multistep}) and a lower bound for one-step error (Lemma
\ref{lem:onestepbound}). 
\begin{theorem} \label{thm:tightness}
  If $P$ is deterministic, then for any blind policy $\pi$ and any
  state-action distribution $\xi$,
\begin{align*}
\epsilon_{val}^{\xi, \pi, T} &\le\ M\sum_{t = 1}^{T} \gamma^{t-1}
                             \expect_{(s, a) \sim \xi}\big[\|D^t_{s, a,
                             \pi} - \hat{D}^t_{s, a, \pi}\|_1\big]\\
&\le
 2M \sum_{t = 1}^{T-1}\gamma^t
  \expect_{(s, z, a) \sim H^{t}_{\xi, \pi}} \big[1 -
  \hat{P}_z^a(\sigma^a_s)\big]\\
&\le \frac{2M}{1 - \gamma}\sum_{t = 1}^{T-1} (\gamma^{t} - 
\gamma^{T})\expect_{(s, a) \sim
  D^t_{\xi, \pi}}\big[1 - \hat{P}_s^a(\sigma^a_s)\big].
\end{align*}
\end{theorem}
Thus, with a deterministic environment and a blind rollout policy, the
hallucinated one-step error of the model is more tightly related to
MBRL performance than the standard one-step error. This is the
theoretical reason for the empirical success of Hallucinated Replay
\citep{talvitie2014model}, which trains the model to predict the next
environment state, given its own samples as input. We now exploit this
fact to develop a novel MBRL algorithm that similarly uses
hallucinated training to mitigate the impact of model class
limitations and that offers strong theoretical guarantees.

\section{Hallucinated DAgger-MC} \label{sec:hdaggermc}

The ``Data Aggregator'' (DAgger) algorithm \citep{ross2012agnostic} was
the first practically implementable MBRL algorithm with performance
guarantees agnostic to the model class. It did, however, require that
the planner be near optimal. DAgger-MC \citep{talvitie2015agnostic}
relaxed this assumption, accounting for the limitations of the planner
that uses the model (one-ply MC). This section augments
DAgger-MC to use hallucinated training, resulting in
the Hallucinated DAgger-MC algorithm\footnote{``Is this a dagger which I see before me,\\
The handle toward my hand? Come, let me clutch thee.\\
I have thee not, and yet I see thee still.\\
Art thou not, fatal vision, sensible\\
To feeling as to sight? Or art thou but\\
A dagger of the mind, a false creation,\\
Proceeding from the heat-oppress'd brain?'' [{\em Macbeth}
2.1.33--39]}, or H-DAgger-MC (Algorithm \ref{alg:HDAggerMC}).

In addition to assuming a particular form for the planner (one-ply MC
with a blind rollout policy), H-DAgger-MC assumes that the model will
be ``unrolled'' (similar to,
e.g. \citenop{Abbeel:2006:UIM:1143844.1143845}). Rather than learning
a single model $\hat{P}$, H-DAgger-MC learns a set of models
$\{\hat{P}^1, \ldots, \hat{P}^{T-1}\} \subseteq \mathcal{C}$, where
model $\hat{P}^i$ is responsible for predicting the outcome of step
$i$ of a rollout, given the state sampled from $\hat{P}^{i-1}$ as
input. The importance of learning an unrolled model will be
discussed more deeply in Section \ref{sec:unrolled}.

Much of the H-DAgger-MC algorithm is identical to DAgger-MC. The main
difference lies in lines \ref{line:startroll}-\ref{line:endroll}, in
which $\rho$ is executed in both the environment and the model to
generate hallucinated examples. This trains the model to self-correct
during rollouts. Like DAgger and DAgger-MC, H-DAgger-MC requires the
ability to reset to the initial state distribution
$\mu$ and also the ability to reset to an ``exploration distribution''
$\nu$. The exploration distribution ideally ensures that the agent
will encounter states that would be visited by a good policy,
otherwise no agent could promise good performance. The performance
bound for H-DAgger-MC will depend in part on the quality of the
selected $\nu$.

\begin{algorithm}[t]
  \caption{Hallucinated DAgger-MC}\label{alg:HDAggerMC}
  \begin{algorithmic}[1]
    \Require Exploration distribution $\nu$, \textsc{Online-Learner}{},
    \Call{MC-Planner}{} (blind rollout policy $\rho$, rollout depth
    $T$), num. iterations $N$, num. rollouts per iteration $K$
    \State Get initial datasets $\mathcal{D}^{1:T-1}_1$ (maybe using $\nu$)
    \State Initialize $\hat{P}^{1:T-1}_1 \gets$ \Call{Online-Learner}{$\mathcal{D}^{1:T-1}_1$}.
    \State Initialize $\hat{\pi}_1 \gets$ \Call{MC-Planner}{$\hat{P}^{1:T-1}_1$}.
    \For{$n \gets 2 \ldots N$}
    \For{$k \gets 1 \ldots K$}
    \State With probability... \label{line:startxi} \hfill $\triangleright$ Sample $(x, b) \sim \xi_n$...
    \State \hspace{\algorithmicindent} $\nicefrac{1}{2}$: Sample $(x, b) \sim D_{\mu}^{\hat{\pi}_{n-1}}$
    \State \hspace{\algorithmicindent} $\nicefrac{1}{4}$: Reset to $(x, b) \sim \nu$.
    \State \hspace{\algorithmicindent} $\nicefrac{(1-\gamma)}{4}$: Sample $x \sim \mu$, $b \sim
     \hat{\pi}_{n-1}(\cdot \mid x)$. 
     \State \hspace{\algorithmicindent} $\nicefrac{\gamma}{4}$: 
     \State \hspace{\algorithmicindent}\hspace{\algorithmicindent} Reset to $(y, c)
     \sim \nu$
     \State \hspace{\algorithmicindent}\hspace{\algorithmicindent}
     Sample $x \sim P(\cdot \mid y, c)$, $b \sim \hat{\pi}_{n-1}(\cdot \mid x)$
    \State Let $s \gets x$, $z \gets x$, $a \gets b$.\label{line:endxi}
    \For{$t \leftarrow 1\ldots T-1$} \label{line:startroll}\hfill
    $\triangleright$ Sample from $H^t_n$...
    \State Sample $s' \sim P(\cdot \mid s, a)$.
    \State Add $\langle z, a, s' \rangle$ to
    $\mathcal{D}^t_n$. \hfill $\triangleright$ Hallucinated training
    \Statex \hfill $\triangleright$ (DAgger-MC adds $\langle s, a, s' \rangle$ instead).
    \State Sample $z' \sim \hat{P}^t_{n-1}(\cdot \mid z, a)$.
    \State Let $s \gets s'$, $z \gets z'$, and sample $a \sim \rho$.\label{line:endroll}
    \EndFor
    \EndFor
    \State $\hat{P}^{1:T-1}_n \gets$
    \Call{Online-Learner}{$\hat{P}^{1:T-1}_{n-1}$,
      $\mathcal{D}^{1:T-1}_n$} 
    \State $\hat{\pi}_n \gets$ \Call{MC-Planner}{$\hat{P}^{1:T-1}_n$}.
    \EndFor
    \State \Return the sequence $\hat{\pi}_{1:N}$
  \end{algorithmic}
\end{algorithm}

We now analyze H-DAgger-MC, adapting \citenoun{ross2012agnostic}'s DAgger
analysis. Let $H^t_n$ be the distribution from which H-DAgger-MC
samples a training triple at depth $t$ (lines
\ref{line:startxi}-\ref{line:endxi} to pick an initial state-action
pair, lines \ref{line:startroll}-\ref{line:endroll} to
roll out). Define the error of the model at depth $t$ to be
$\bar{\epsilon}^{t}_{prd} = \frac{1}{N}\sum_{n = 1}^N \expects_{(s, z,
  a) \sim H^t_n} [1 -\hat{P}_{n}^{t}(\sigma_{s}^{a} \mid z, a)]$.

For a policy $\pi$, let
$c_{\nu}^{\pi} = \sup_{s,a} \frac{D_{\mu, \pi}(s,a)}{\nu(s,a)}$
represent the mismatch between the discounted state-action
distribution under $\pi$ and the exploration distribution $\nu$. Now,
consider the sequence of policies $\hat{\pi}_{1:N}$ generated by
H-DAgger-MC. Let $\bar{\pi}$ be the uniform mixture over all policies
in the sequence. Let
$\bar{\epsilon}_{mc} = \frac{1}{N} \frac{4}{1 - \gamma} \sum_{n=1}^N
(\|\bar{Q}_n - \hat{Q}^{\rho}_{T,n}\|_{\infty} + \frac{2}{1 - \gamma} \|BV^{\rho}_T -
V^{\rho}_T\|_{\infty} $
be the error induced by the choice of planning algorithm, averaged
over all iterations.
\begin{lem} \label{lem:prederror}
  In H-DAgger-MC, the policies $\hat{\pi}_{1:N}$ are such that for any
  policy $\pi$,
\begin{align*}
 \expect_{s
  \sim \mu}\big[V^{\pi}(s) - V^{\bar{\pi}}(s)\big] \le \frac{8
  M}{1-\gamma} c_{\nu}^{\pi} \sum_{t = 1}^{T-1}
  \bar{\epsilon}^t_{prd} + \bar{\epsilon}_{mc}.
\end{align*}
\end{lem}
Note that this result holds for {\em any} comparison policy
$\pi$. Thus, if $\bar{\epsilon}_{mc}$ is small and the learned models have
low hallucinated one-step prediction error, then if $\nu$ is similar
to the state-action distribution under {\em some} good policy,
$\bar{\pi}$ will compare favorably to it. Like the original DAgger and
DAgger-MC results, Lemma \ref{lem:prederror} has limitations. It uses
the L1 loss, which is not always a practical learning objective. It
also assumes that the expected loss at each iteration can be computed
exactly (i.e. that there are infinitely many samples per
iteration). It also applies to the average policy $\bar{\pi}$, rather
than the last policy in the sequence. \citenoun{ross2012agnostic}
discuss extensions that address more practical loss functions, finite
sample bounds, and results for $\hat{\pi}_N$.

The next question is, of course, when will the learned models be
accurate? Following \citenoun{ross2012agnostic} note that
$\bar{\epsilon}^t_{prd}$ can be interpreted as the average loss of an
online learner on the problem defined by the aggregated datasets at
each iteration. In that case, for each horizon depth $t$ let
$\bar{\epsilon}^{t}_{mdl}$ be the error of the best model in $\mathcal{C}$
under the training distribution at that depth, in
retrospect. Specifically,
$\bar{\epsilon}^{t}_{mdl} = \inf_{P' \in
  \mathcal{C}}\frac{1}{N}\sum_{n = 1}^N \expects_{(s, z, a) \sim
  H^t_n} [1 - P'(\sigma_{s}^{a} \mid z, a)]$.
Then the average regret for the model at depth $t$ is
$\bar{\epsilon}^t_{rgt} = \bar{\epsilon}_{prd}^t -
\bar{\epsilon}_{mdl}^t$.
For a no-regret online learning algorithm,
$\bar{\epsilon}_{rgt}^t \rightarrow 0$ as $N \rightarrow \infty$. This
gives the following bound on H-DAgger-MC's performance in terms of model regret.
\begin{theorem}\label{thm:hdaggermc}
  In H-DAgger-MC, the policies $\hat{\pi}_{1:N}$ are such that for any
  policy $\pi$,
\begin{align*}
  \expect_{s
  \sim \mu}\big[V^{\pi}(s) - V^{\bar{\pi}}(s)\big] \le \frac{8
  M}{1-\gamma} c_{\nu}^{\pi} \sum_{t = 1}^{T-1}
  (\bar{\epsilon}_{mdl}^t + \bar{\epsilon}_{rgt}^t) + \bar{\epsilon}_{mc},
\end{align*}
and if the model learning algorithm is no-regret then
$\bar{\epsilon}_{rgt}^t \rightarrow 0$ as $N \rightarrow \infty$ for
each $1 \le t \le T-1$.
\end{theorem}

Theorem \ref{thm:hdaggermc} says that if $\mathcal{C}$ contains a
low-error model for each rollout depth then low error models will be
learned. Then, as discussed above, if $\bar{\epsilon}_{mc}$ is small
and $\nu$ visits important states, the resulting policy will yield
good performance.  Notably, even with hallucinated training, if
$\mathcal{C}$ contains a perfect model, H-DAgger-MC will learn a
perfect model.

It is important to note that this result does {\em not} promise that
H-DAgger-MC will eventually achieve the performance of the best
performing set of models in the class. The model at each rollout depth
is trained to minimize prediction error given the input distribution
provided by the shallower models. Note, however, that changing the
parameters of a model at one depth alters the training distribution
for deeper models. It is possible that better overall error could be
achieved by {\em increasing} the prediction error at one depth in
exchange for a favorable state distribution for deeper models. This
effect is not taken into account by H-DAgger-MC.

\section{Empirical Illustration}

In this section we illustrate the practical impact of optimizing
hallucinated error by comparing DAgger, DAgger-MC, and H-DAgger-MC in
the Shooter example described in Section
\ref{sec:onestep}\footnote{Source code for these experiments may be
  found at
  \url{github.com/etalvitie/hdaggermc}.}. The experimental setup
matches that of \citenoun{talvitie2015agnostic} for comparison's sake,
though the qualitative comparison presented here is robust to the
parameter settings.

In all cases one-ply MC was used with 50 uniformly random rollouts of
depth 15 at every step. The exploration distribution was generated by
following the optimal policy with $(1 - \gamma)$ probability of
termination at each step. The model for each pixel was learned using
Context Tree Switching \citep{veness2012context}, similar to the
FAC-CTW algorithm \citep{veness11montecarlo}, and used a $7\times 7$
neighborhood around the pixel in the previous timestep as input. Data
was shared across all positions. The discount factor was
$\gamma = 0.9$. In each iteration 500 training rollouts were generated
and the resulting policy was evaluated in an episode of length 30. The
discounted return obtained by the policy in each iteration is
reported, averaged over 50 trials.

\begin{figure*}
\centering
\includegraphics[width=0.96\linewidth]{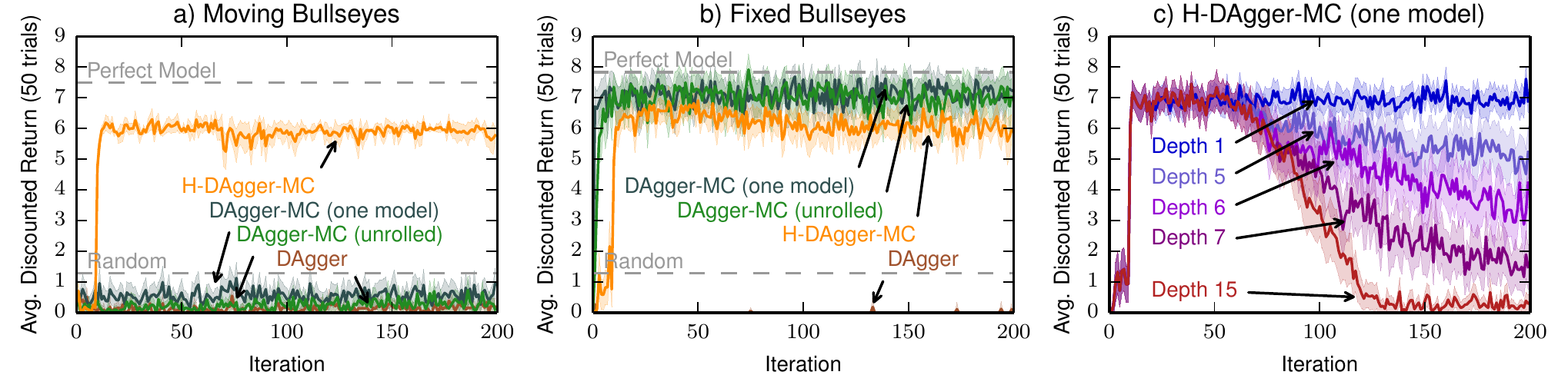}
\caption{a,b) Comparing DAgger, DAgger-MC, and H-DAgger-MC in Shooter with moving and fixed
  bullseyes, respectively. c) H-DAgger-MC in Shooter with fixed
  bullseyes using a single model across time steps, truncating rollouts at various depths.}
\label{fig:shooterResults}
\end{figure*}

The results can be seen in Figure \ref{fig:shooterResults}a and
\ref{fig:shooterResults}b. The shaded regions represent 95\%
confidence intervals for the mean performance. The benchmark lines
labeled ``Random'' and ``Perfect Model'' represent the average
performance of the uniform random policy and one-ply Monte Carlo using
a perfect model, respectively. In Figure \ref{fig:shooterResults}a the
bullseyes move, simulating the typical practical reality that
$\mathcal{C}$ does not contain a perfect model. In Figure
\ref{fig:shooterResults}b the bullseyes have fixed positions, so
$\mathcal{C}$ {\em does} contain a perfect model.

As observed by \citenoun{talvitie2015agnostic}, DAgger performs poorly
in both versions, due to the suboptimal planner. DAgger-MC is able to
perform well with fixed bullseyes (Figure \ref{fig:shooterResults}b),
but with moving bullseyes the model suffers from compounding
errors and is not useful for planning (Figure
\ref{fig:shooterResults}a). This holds for a single model and for an
``unrolled'' model.

In these experiments one practically-minded alteration was made to the
H-DAgger-MC algorithm. In early training the model is highly
inaccurate, and thus deep rollouts produce incoherent
samples. Training with these samples is counter-productive (also, the
large number of distinct, nonsensical contexts renders CTS
impractical). For these experiments, training rollouts in iteration
$i$ were truncated at depth $\lfloor \nicefrac{i}{10} \rfloor$.
Planning rollouts in these early iterations use the models that have
been trained so far and then repeatedly apply the deepest model in
order to complete the rollout. \citenoun{talvitie2014model},
\citenoun{venkatraman2015improving}, and \citenoun{oh2015action} all
similarly discarded noisy examples early in training. This transient
modification does not impact H-DAgger-MC's asymptotic guarantees.

In Figure \ref{fig:shooterResults}a it is clear that H-DAgger-MC
obtains a good policy despite the limitations of the model
class. Hallucinated training has made MBRL possible with both a flawed model
and a flawed planner while the standard approach has failed
entirely. In the case that $\mathcal{C}$ contains a perfect model
(rare in problems of genuine interest) H-DAgger-MC is outperformed by
DAgger-MC. Despite the adjustment to training, deep models still
receive noisy inputs. Theoretically the model should become perfectly
accurate in the limit, though in practice it may do so very slowly.

\subsection{Impact of the unrolled model}\label{sec:unrolled}

Recall that the H-DAgger-MC algorithm assumes the model will be
``unrolled,'' with a separate model responsible for sampling each step
in a rollout. This has clear practical disadvantages, but it is
important theoretically. When one model is used across all time-steps,
convergence to a perfect model cannot be guaranteed, even if one
exists in $\mathcal{C}$.

In Figure \ref{fig:shooterResults}c, H-DAgger-MC has been trained
using a single model in Shooter with fixed bullseyes. The temporary
truncation schedule described above is employed, but the training
rollouts have been permanently limited to various depths. First
consider the learning curve marked ``Depth 15'', where training rollouts are
permitted to reach the maximum depth. While the rollouts are
temporarily truncated the model does well, but performance degrades as
longer rollouts are permitted {\em even though $\mathcal{C}$ contains
  a perfect model}!

Recall from Section \ref{sec:hdaggermc} that changing the model
parameters impacts both prediction error and the future training
distribution. Furthermore, training examples generated by deep
rollouts may contain highly flawed samples as inputs. Sometimes
attempting to ``correct'' a large error (i.e. reduce prediction error)
causes additional, even worse errors in the next iteration (i.e. harms
the training distribution).  For instance consider a hallucinated
training example with the 4th screen from Figure \ref{fig:shooter}b as
input and the 5th screen from Figure \ref{fig:shooter}a as the
target. The model would effectively learn that targets can appear out
of nowhere, an error that would be even harder to correct in future
iterations. With a single model across timesteps, a feedback loop can
emerge: the model parameters change to attempt to correct large
errors, thereby causing larger errors, and so on. This feedback loop
causes the observed performance crash. With an unrolled model the
parameters of each sub-model cannot impact that
sub-model's own training distribution, ensuring stability.

Note that none of \citenoun{talvitie2014model},
\citenoun{venkatraman2015improving}, or \citenoun{oh2015action} used
an unrolled model. As such, all of their approaches are subject to
this concern. Notably, all three limited the depth of training
rollouts, presumably to prevent overly noisy samples. Figure
\ref{fig:shooterResults}c shows that in this experiment, the shorter
the training rollouts, the better the performance. These results show
that it may be possible in practice to avoid unrolling the model by
truncating training rollouts, though for now there is no performance
guarantee or principled choice of rollout depth.

\section{Conclusions and future work}

The primary contribution of this work is a deeper theoretical
understanding of how to perform effective MBRL in the face of model
class limitations. Specifically we have examined a novel measure of
model quality that, under some assumptions, is more tightly related to
MBRL performance than standard one-step prediction error. Using this
insight, we have also analyzed a MBRL algorithm that achieves good
control performance despite flaws in the model and planner and
provides strong theoretical performance guarantees.

We have also seen negative results indicating that hallucinated
one-step error may not be an effective optimization criterion in the
most general setting. This poses the open challenge of relaxing the
assumptions of deterministic dynamics and blind policies, or of
developing alternative approaches for improving multi-step error in
more general settings. We have further observed that hallucinated
training can cause stability issues, since model parameters affect
both prediction error and the training distribution itself. It would
be valuable to develop techniques that account for both of these
effects when adapting model parameters.

Specializing to the one-ply MC planning algorithm may seem
restrictive, but then again, the choice of planning algorithm cannot
make up for a poor model. When the model class is limited, H-DAgger-MC
is likely still a good choice over DAgger, even with a more
sophisticated planner. Still, it would be valuable to investigate
whether these principles can be applied to more sophisticated planning
algorithms.

Though this work has assumed that the reward function is known, the
results presented here can be straightforwardly extended to account
for reward error. However, this also raises the interesting point that
sampling an ``incorrect'' state has little negative impact if the
sampled state's rewards and transitions are similar to the ``correct''
state. It may be possible to exploit this to obtain
still tighter bounds, and more effective guidance for model learning
in MBRL architectures.

\section*{Acknowledgements}

This work was supported in part by NSF grant IIS-1552533. Many thanks
to Marc Bellemare whose feedback has positively influenced the work,
both in substance and presentation. Thanks to Drew Bagnell and Arun
Venkatraman for their valuable insights. Thanks also to Joel Veness
for his freely available FAC-CTW and CTS implementations
(\url{http://jveness.info/software/}).

{\small \bibliography{hdaggermc_AAAI2017}}
\bibliographystyle{abbrvnat}
\appendix

\section*{Appendix: Proofs}
\allowdisplaybreaks

\setcounter{lem}{1}

\begin{lem}
For any policy $\pi$ and
  state-action distribution $\xi$,
\begin{align*}
\epsilon_{val}^{\xi, \pi, T} \le \frac{M}{1 - \gamma}\sum_{t = 1}^{T-1} (\gamma^{t} - 
\gamma^{T})\expect_{(s, a) \sim
  D^t_{\xi, \pi}}\big[\|P_s^a - \hat{P}_s^a\|_1\big].
\end{align*}
\end{lem}
\begin{proof}
First note that
\begin{align*}
&\epsilon_{val}^{\xi, \pi, T} = \expect_{(s_1, a_1) \sim \xi}\big[|\hat{Q}_T^{\pi}(s_1, a_1) - Q_T^{\pi}(s_1,
  a_1)|\big]&\\ 
&= \gamma \expect_{(s_1, a_1) \sim \xi}\big[|\expect_{s \sim
  \hat{P}_{s_1}^{a_1}}[\hat{V}_{T-1}^{\pi}(s)] - \expect_{s \sim P_{s_1}^{a_1}}[V_{T-1}^{\pi}(s)]|\big]&\\
&= \gamma \expect_{(s_1, a_1) \sim \xi}\Big[\big|\expect_{s \sim
  \hat{P}_{s_1}^{a_1}}[\hat{V}_{T-1}^{\pi}(s)] - \expect_{s \sim P_{s_1}^{a_1}}[\hat{V}_{T-1}^{\pi}(s)]&\\
&\hspace{1in} + \expect_{s \sim P_{s_1}^{a_1}}[\hat{V}_{T-1}^{\pi}(s)] - \expect_{s \sim P_{s_1}^{a_1}}[V_{T-1}^{\pi}(s)]\big|\Big]&\\
&\le \gamma \expect_{(s_1, a_1) \sim \xi}
    \Big[\big|\expect_{\substack{s \sim \hat{P}_{s_1}^{a_1}\\ a \sim
  \pi_{s}}}[\hat{Q}_{T-1}^{\pi}(s, a)] - \expect_{\substack{s \sim P_{s_1}^{a_1}\\ a \sim \pi_{s}}}[\hat{Q}_{T-1}^{\pi}(s, a)]\big|\Big]&\\
&\hspace{.3in} + \gamma \expect_{(s, a) \sim D^2_{\xi,
  \pi}}\big[|\hat{Q}_{T-1}^{\pi}(s, a) - Q_{T-1}^{\pi}(s, a)])|\big].&
\end{align*}

Rolling out the recurrence gives
\begin{align*}
&\epsilon_{val}^{\xi, \pi, T} = \expect_{(s_1, a_1) \sim \xi}\big[|\hat{Q}_T^{\pi}(s_1, a_1) - Q_T^{\pi}(s_1,
  a_1)|\big]&\\
& \le \sum_{t = 1}^{T-1} \gamma^{t} \expect_{(s, a) \sim
  D^t_{\xi, \pi}}\Big[\big|\expect_{\substack{s' \sim \hat{P}_s^a\\a' \sim
  \pi_{s'}}}[\hat{Q}_{T-t}^{\pi}(s', a')]&\\
&\hspace{1.8in} - \expect_{\substack{s' \sim P_s^a\\a' \sim \pi_{s'}}}[\hat{Q}_{T-t}^{\pi}(s', a')]\big|\Big]&\\
&\le \sum_{t = 1}^{T-1} \gamma^{t} \sum_{s', a'} \pi_{s'}(a')
  \hat{Q}^{\pi}_{T-t}(s', a')&\\
&\hspace{1.3in} \expect_{(s, a) \sim
  D^t_{\xi, \pi}}\big[|\hat{P}_s^a(s') - P_s^a(s')|\big]&\\
&\le \sum_{t = 1}^{T-1} \gamma^{t} \frac{M(1 - 
\gamma^{T-t})}{1 - \gamma}\expect_{(s, a) \sim
  D^t_{\xi, \pi}}\big[\|\hat{P}_s^a - P_s^a\|_1\big]&\\
&= \frac{M}{1 - \gamma}\sum_{t = 1}^{T-1} (\gamma^{t} - 
\gamma^{T})\expect_{(s, a) \sim
  D^t_{\xi, \pi}}\big[\|\hat{P}_s^a - P_s^a\|_1\big].&\qedhere
\end{align*}
\end{proof}

\begin{lem}
  For any policy $\pi$ and state-action distribution $\xi$,
\begin{align*}
\epsilon_{val}^{\xi, \pi, T} \le M\sum_{t = 1}^{T} \gamma^{t-1}
                             \expect_{(s, a) \sim \xi}\big[\|D^t_{s, a,
                             \pi} - \hat{D}^t_{s, a, \pi}\|_1\big].
\end{align*}
\end{lem}
\begin{proof}
This follows straightforwardly from the definition.
\begin{align*}
&\epsilon_{val}^{\xi, \pi, T} =\expect_{(s, a) \sim \xi} \big[|Q_T^\pi(s, a) - \hat{Q}_T^\pi(s,a)|\big]&\\
  &= \expect_{(s, a) \sim \xi} \Bigg[\bigg| \sum_{t = 1}^T \gamma^{t-1}\sum_{(s',
    a')}(D^t_{s, a, \pi}(s', a')&\\
  &\hspace{1.7in}- \hat{D}^t_{s, a, \pi}(s', a') ) R(s',
    a') \bigg| \Bigg]&\\
  &\le \sum_{t = 1}^T \gamma^{t-1}\expect_{(s, a) \sim \xi} \bigg[\sum_{(s',
    a')} \big| (D^t_{s, a, \pi}(s', a')&\\
&\hspace{1.7in} - \hat{D}^t_{s, a, \pi}(s', a') ) R(s',
    a') \big| \bigg]&\\
  &\le M\sum_{t = 1}^T \gamma^{t-1}\expect_{(s, a) \sim \xi} \bigg[\sum_{(s',
    a')} \big| D^t_{s, a, \pi}(s', a')&\\
&\hspace{2.2in} - \hat{D}^t_{s, a, \pi}(s', a')
  \big| \bigg]&\\
&= M\sum_{t = 1}^{T} \gamma^{t-1}
                             \expect_{(s, a) \sim \xi}\big[\|D^t_{s, a,
                             \pi} - \hat{D}^t_{s, a, \pi}\|_1\big].&\qedhere
\end{align*}
\end{proof}

\begin{lem}
  For any policy $\pi$ and state-action distribution $\xi$,
\begin{align*}
\epsilon_{val}^{\xi, \pi, T}\le
  M \sum_{t = 1}^{T-1} \gamma^t \expect_{(s, a, z, b) \sim J^{t}_{\xi, \pi}}
  \big[\|P_s^a - \hat{P}_z^b\|_1\big].
\end{align*}
\end{lem}
\begin{proof}
First note that for any $s_1, a_1, s', a'$ and for $t > 1$,
\begin{align*}
&D^t_{s_1,a_1, \pi}(s', a')- \hat{D}^t_{s_1,a_1, \pi}(s', a')&\\
 &= \sum_{s, a}
  D^{t-1}_{s_1, a_1, \pi} (s, a) P_s^a(s') \pi_{s'}(a')&\\
&\hspace{1in} - \sum_{z,
  b} \hat{D}^{t-1}_{s_1, a_1, \pi} (z, b) \hat{P}_z^b(s') \pi_{s'}(a')&\\
&= \pi_{s'}(a') \bigg(\sum_{s, a}
  D^{t-1}_{s_1, a_1, \pi} (s, a) P_s^a(s') \sum_{z, b}
  \hat{D}^{t-1}_{s_1, a_1, \pi}(z, b)&\\
&\hspace{.6in} - \sum_{z,
  b} \hat{D}^{t-1}_{s_1, a_1, \pi} (z, b) \hat{P}_z^b(s') \sum_{s, a} D^{t-1}_{s_1, a_1, \pi}(s, a)\bigg)\\
&= \pi_{s'}(a')\bigg(\sum_{s, a, z, b}
  D^{t-1}_{s_1, a_1, \pi}(s, a)\hat{D}^{t-1}_{s_1, a_1, \pi}(z, b)&\\
&\hspace{2.1in}  (P_s^a(s') - \hat{P}_z^b(s'))\bigg).&
\end{align*}

Further note that for any $s, a, \pi$,
$D^1_{s, a, \pi} = \hat{D}^1_{s, a, \pi}$. Combining these facts with
Lemma \ref{lem:multistep} we see that
\begin{align*}
&\epsilon_{val}^{\xi, \pi, T} = \expect_{(s, a) \sim \xi} \big[|Q_T^\pi(s, a) -
  \hat{Q}_T^\pi(s,a)|\big]&\\
& \le M\sum_{t = 1}^{T}\gamma^{t-1}
                             \expect_{(s, a) \sim \xi}\big[\|D^t_{s, a,
                             \pi} - \hat{D}^t_{s, a, \pi}\|_1\big]&\\
& = M\sum_{t = 2}^{T}\gamma^{t-1}
                             \expect_{(s_1, a_1) \sim \xi}\bigg[\sum_{s',
  a'} \Big|D^t_{s_1, a_1,
                             \pi}(s', a')&\\
&\hspace{2.1in} - \hat{D}^t_{s_1, a_1, \pi}(s',
  a')\Big|\bigg]&\\
& = M\sum_{t = 1}^{T-1}\gamma^{t}
                             \expect_{(s_1, a_1) \sim \xi}\bigg[\sum_{s',
  a'} \Big|D^{t+1}_{s_1, a_1,
                             \pi}(s', a')&\\
&\hspace{2.1in} - \hat{D}^{t+1}_{s_1, a_1, \pi}(s',
  a')\Big|\bigg]&\\
& \le M\sum_{t = 1}^{T-1}\gamma^{t}
                             \expect_{(s_1, a_1) \sim \xi}\bigg[\sum_{s, a, z, b}
  D^{t}_{s_1, a_1, \pi}(s, a)\hat{D}^{t}_{s_1, a_1, \pi}(z, b)&\\
&\hspace{1.95in}  \sum_{s'}\Big|P_s^a(s') - \hat{P}_z^b(s')\Big|\bigg]&\\
& = M\sum_{t = 1}^{T-1}\gamma^{t}
                            \sum_{s, a, z, b} \|P_s^a -
                            \hat{P}_z^b\|_1&\\
&\hspace{1.1in} \expect_{(s_1, a_1) \sim \xi}\big[
  D^{t}_{s_1, a_1, \pi}(s, a)\hat{D}^{t}_{s_1, a_1, \pi}(z,
  b)\big]&\\
&=  M \sum_{t = 1}^{T-1} \gamma^t \expect_{(s, a, z, b) \sim J^{t}_{\xi, \pi}}
  \big[\|P_s^a - \hat{P}_z^b\|_1\big].&\qedhere
\end{align*}
\end{proof}

\setcounter{lem}{6}
\begin{lem}
  If $P$ is deterministic, then for any blind policy $\pi$ and any
  state-action distribution $\xi$,
\begin{align*}
\epsilon_{val}^{\xi, \pi, T} \le 2M \sum_{t = 1}^{T-1}\gamma^t
  \expect_{(s, z, a) \sim H^{t}_{\xi, \pi}} \big[1 - \hat{P}_z^a(\sigma^a_s)\big].
\end{align*}
\end{lem}
\begin{proof}
  As in the proof of Lemma \ref{lem:hallucinated} note that for any
  $s_1, a_1, \pi$, $D^1_{s_1, a_1, \pi} = \hat{D}^1_{s_1, a_1, \pi}$. Further note
  that under these assumptions, for any $s_1, a_1$, $\|D^2_{s_1,a_1, \pi}- \hat{D}^2_{s_1,a_1,
    \pi}\big\|_1$
\begin{align*}
 &= \sum_{s, a}\pi(a)\big|P_{s_1}^{a_{1}}(s)-
   \hat{P}_{s_1}^{a_{1}}(s)\big| = \sum_{s}\big|P_{s_1}^{a_1}(s) - \hat{P}_{s_1}^{a_{1}}(s)\big|.
\end{align*}
For $t>2$,   $\|D^t_{s_1,a_1, \pi} - \hat{D}^t_{s_1,a_1, \pi}\|_1$
\begin{align*}
  &= \sum_{s_t}\bigg|\sum_{a_{2:t}} \pi(a_{2:t} \mid a_1) \bigg( \sum_{s}
    P_{s_1}^{a_{1:t-2}}(s) P_s^{a_{t-1}}(s_t)&\\
& \hspace{1.6in}- \sum_{z} \hat{P}_{s_1}^{a_{1:t-2}}(z) \hat{P}_{z}^{a_{t-1}}(s_t)\bigg)\bigg|&\\
  &= \sum_{s_t} \bigg|\sum_{a_{2:t-1}} \pi(a_{2:t-1} \mid a_1)&\\
&\hspace{.6in} \bigg( \sum_{s}
    P_{s_1}^{a_{1:t-2}}(s) P_s^{a_{t-1}}(s_t)\sum_{z}\hat{P}_{s_1}^{a_{1:t-2}}(z)\\
  &\hspace{.7in}- \sum_{z} \hat{P}_{s_1}^{a_{1:t-2}}(z) \hat{P}_z^{a_{t-1}}(s_t)\sum_{s} P_{s_1}^{a_{1:t-2}}(s)\bigg)\bigg|\\
  &= \sum_{s_t}\bigg| \sum_{a_{2:t-1}} \pi(a_{2:t-1} \mid a_1) \sum_{s, z}
  P_{s_1}^{a_{1:t-2}}(s)\hat{P}_{s_1}^{a_{1:t-2}}(z)&\\
&\hspace{1.8in} \Big(P_s^{a_{t-1}}(s_t) - \hat{P}_{z}^{a_{t-1}}(s_t)\Big)\bigg|\\
  &= \sum_{s_t}\bigg|\sum_{a_{2:t-1}} \pi(a_{2:t-1} \mid a_1) \sum_{s, z} P_{s_1}^{a_{1:t-2}}(s) \hat{P}_{s_1}^{a_{1:t-2}}(z)&\\
&\hspace{1.8in} \Big(P_s^{a_{t-1}}(s_t) - \hat{P}_z^{a_{t-1}}(s_t)\Big)\bigg|\\
  &\le \sum_{a_{2:t-1}} \pi(a_{2:t-1}) \sum_{s, z} P_{s_1}^{a_{1:t-2}}(s)\hat{P}_{s_1}^{a_{1:t-2}}(z)&\\
&\hspace{1.6in} \sum_{s_t}\Big|P_{s}^{a_{t-1}}(s_t) - \hat{P}_z^{a_{t-1}}(s_t)\Big|.
\end{align*}

Now, for any $s$, $z$, $a$, $\sum_{s_t}\Big|P_s^a(s_t) - \hat{P}_z^a(s_t)\Big|$
\begin{align*}
& = \big(1 - \hat{P}_z^a(\sigma_s^a)\big) + \sum_{s_t \ne
  \sigma_s^a}\big|-\hat{P}_z^a(s_t)\big|&\\
& = 2\big(1 - \hat{P}_z^a(\sigma_s^a)\big).&
\end{align*}

Combining these facts with Lemma \ref{lem:multistep},
\begin{align*}
&\epsilon_{val}^{\xi, \pi, T} = \expect_{(s, a) \sim \xi} \big[|Q_T^\pi(s, a) -
  \hat{Q}_T^\pi(s,a)|\big]&\\
  & \le M\sum_{t = 1}^{T} \gamma^{t-1} \expect_{(s, a) \sim
    \xi}\big[\|D^t_{s, a,
    \pi} - \hat{D}^t_{s, a, \pi}\|_1\big]&\\
  & = M\sum_{t = 2}^{T} \gamma^{t-1} \expect_{(s_1, a_1) \sim
    \xi}\bigg[\|D^t_{s_1, a_1, \pi}- \hat{D}^t_{s_1, a_1, \pi}\|_1\bigg]&\\
  & = M \gamma \expect_{(s_1, a_1) \sim
    \xi}\bigg[\|D^2_{s_1, a_1, \pi} -
  \hat{D}^2_{s_1, a_1, \pi}\|_1\bigg]&\\ 
  &\hspace{0.15in}+M\sum_{t = 3}^{T} \gamma^{t-1} \expect_{(s_1, a_1) \sim
    \xi}\bigg[\|D^t_{s_1, a_1, \pi}- \hat{D}^t_{s_1, a_1, \pi}\|_1\bigg]&\\
  & = M \gamma \expect_{(s_1, a_1) \sim
    \xi}\bigg[\|D^2_{s_1, a_1, \pi} -
  \hat{D}^2_{s_1, a_1, \pi}\|_1\bigg]&\\ 
  &\hspace{0.15in}+M\sum_{t = 2}^{T-1} \gamma^{t} \expect_{(s_1, a_1) \sim
    \xi}\bigg[\sum_{s', a'} \|D^{t+1}_{s_1, a_1, \pi} -
  \hat{D}^{t+1}_{s_1, a_1, \pi}\|_1\bigg]&\\
  & \le 2M\gamma \expect_{(s_1, a_1) \sim
    \xi}\big[1 -
  \hat{P}_{s_1}^{a_1}(\sigma_{s_1}^{a_1})\big]\\
  &\hspace{0.15in}+2M\sum_{t = 2}^{T-1} \gamma^{t} \expect_{(s_1, a_1) \sim
    \xi}\bigg[\sum_{a_{2:t}} \pi(a_{2:t}\mid a_1)&\\
&\hspace{0.8in} \sum_{s, z}
  P_{s_1}^{a_{1:t-1}}(s) \hat{P}_{s_1}^{a_{1:t-1}}(z)\Big(1 -
  \hat{P}_{z}^{a_{t}}(\sigma_{s}^{a_{t}})\Big) \bigg]&\\
  & = 2M\gamma \expect_{(s_1, a_1) \sim
    \xi}\big[1 -
  \hat{P}_{s_1}^{a_1}(\sigma_{s_1}^{a_1})\big]\\
  &\hspace{0.15in}+2M \gamma^{2} \sum_{s_2, z_2, a_2}\Big(1 -
  \hat{P}_{z_2}^{a_2}(\sigma_{s_2}^{a_2})\Big)&\\
&\hspace{1.1in} \expect_{(s_1, a_1) \sim
    \xi}\bigg[\pi(a_2 \mid a_1) P_{s_1}^{a_{1}}(s_2) \hat{P}_{s_1}^{a_{1}}(z_2)\bigg]&\\
  &\hspace{0.15in}+2M\sum_{t = 3}^{T-1} \gamma^{t} \sum_{s_t, z_t, a_t}\Big(1 -
  \hat{P}_{z_t}^{a_t}(\sigma_{s_t}^{a_t})\Big)&\\
&\hspace{0.35in} \expect_{(s_1, a_1) \sim
    \xi}\bigg[\sum_{a_{2:t-1}} \pi(a_{2:t} \mid a_1) P_{s_1}^{a_{1:t-1}}(s_t) \hat{P}_{s_1}^{a_{1:t-1}}(z_t)\bigg]&\\
  & = 2M\gamma \expect_{(s, z, a) \sim
    H^1_{\xi, \pi}}\big[1 -
  \hat{P}_{z}^{a}(\sigma_{s}^{a})\big]\\
  &\hspace{0.15in}+2M \gamma^{2} \expect_{(s, z, a) \sim H^2_{\xi, \pi}}\big[1 -
  \hat{P}_{z}^{a}(\sigma_{s}^{a})\big]&\\
  &\hspace{0.15in}+2M\sum_{t = 3}^{T-1} \gamma^{t} \expect_{(s, z, a) \sim H^t_{\xi, \pi}}\big[1 -
  \hat{P}_{z}^{a}(\sigma_{s}^{a})\big] &\\
&= 2M\sum_{t = 1}^{T-1}\gamma^t
  \expect_{(s, z, a) \sim H^{t}_{\xi, \pi}} \big[1 - \hat{P}(\sigma^a_s \mid
  z, a)\big].&\qedhere
\end{align*}
\end{proof}

\begin{theorem}
  If $P$ is deterministic, then for any blind policy $\pi$ and any
  state-action distribution $\xi$,
\begin{align*}
\epsilon_{val}^{\xi, \pi, T} &\le\ M\sum_{t = 1}^{T} \gamma^{t-1}
                             \expect_{(s, a) \sim \xi}\big[\|D^t_{s, a,
                             \pi} - \hat{D}^t_{s, a, \pi}\|_1\big]\\
&\le
 2M \sum_{t = 1}^{T-1}\gamma^t
  \expect_{(s, z, a) \sim H^{t}_{\xi, \pi}} \big[1 -
  \hat{P}_z^a(\sigma^a_s)\big]\\
&\le \frac{2M}{1 - \gamma}\sum_{t = 1}^{T-1} (\gamma^{t} - 
\gamma^{T})\expect_{(s, a) \sim
  D^t_{\xi, \pi}}\big[1 - \hat{P}_s^a(\sigma^a_s)\big].
\end{align*}
\end{theorem}
\begin{proof}
  The first inequality was proven in Lemma \ref{lem:multistep}. The
  proof of Lemma \ref{lem:detblind} also proves the second
  inequality. Thus, we shall focus on the third inequality. First note that
\begin{align*}
  &\sum_{t = 1}^{T-1} \gamma^t\expect_{(s, z, a) \sim H^{t}_{\xi, \pi}}\big[1 -
  \hat{P}_z^a(\sigma^a_s)\big]&\\
&\hspace{.2in} = \sum_{s, a} \sum_{t=1}^{T-1} \gamma^t H^{t}_{\xi, \pi}(s, s, a)\big[1 -
  \hat{P}_{s}^{a}(\sigma^{a}_{s})\big]&\\
&\hspace{.5in}+\sum_{s, z \ne s, a}
  \sum_{t=1}^{T-1} \gamma^t H^{t}_{\xi, \pi}(s, z, a)\big[1 - \hat{P}_{z}^{a}(\sigma^{a}_{s})\big].&
\end{align*}
This expression has two terms; the first captures prediction error
when the model and environment are in the same state, the second when
the environment and world state differ. Consider the first term.
\begin{align*}
& \sum_{s, a} \sum_{t=1}^{T-1} \gamma^t H^{t}_{\xi,
  \pi}(s, s, a)\big[1 -\hat{P}_{s}^{a}(\sigma^{a}_{s})\big]&\\
& \hspace{.1in} = \sum_{s, a} \Big(1 -
\hat{P}_{s}^{a}(\sigma_{s}^{a})\Big)&\\
&\hspace{.4in} \bigg( \gamma\xi(s, a) + \gamma^2 \expect_{(s_1, a_1) \sim \xi}\pi(a \mid a_1)
P_{s_1}^{a_1}(s)\hat{P}_{s_1}^{a_1}(s)\bigg)&\\
&\hspace{.2in} + \sum_{t = 3}^{T-1}\gamma^t \sum_{s_t, a_t} \Big(1 -
\hat{P}_{s_t}^{a_t}(\sigma_{s_t}^{a_t})\Big)&\\
&\hspace{.3in}\expect_{(s_1, a_1) \sim
                               \xi}\bigg[\sum_{a_{2:t-1}} \pi(a_{2:t}
                             \mid a_1) P_{s_1}^{a_{1:t-1}}(s_t) \hat{P}_{s_1}^{a_{1:t-1}}(s_t) \bigg]&\\
& \hspace{.1in} = \sum_{s, a} \Big(1 -
\hat{P}_{s}^{a}(\sigma_{s}^{a})\Big)&\\
&\hspace{.4in} \bigg( \gamma\xi(s, a) + \gamma^2 \expect_{(s_1, a_1) \sim \xi}\pi(a \mid a_1)
P_{s_1}^{a_1}(s)\bigg)&\\
&\hspace{.2in} + \sum_{t = 3}^{T-1}\gamma^t \sum_{s_t, a_t} \Big(1 -
\hat{P}_{s_t}^{a_t}(\sigma_{s_t}^{a_t})\Big)&\\
&\hspace{.3in}\expect_{(s_1, a_1) \sim
                               \xi}\bigg[\sum_{a_{2:t-1}} \pi(a_{2:t}
                             \mid a_1) P_{s_1}^{a_{1:t-1}}(s_t)\bigg]&\\
& \hspace{.1in} = \sum_{t = 1}^{T-1}\gamma^t
  \expect_{(s, a) \sim D^{t}_{\xi, \pi}} \big[1 - \hat{P}_s^a(\sigma^a_s)\big].&
\end{align*}

Now consider the second term. Note that for any $a$ and $s \ne z$,
$H_{\xi, \pi}^1(s, z, a) = 0$. So,
\begin{align*}
&\sum_{s, z \ne s, a} \sum_{t=1}^{T-1} \gamma^t H^{t}_{\xi,
  \pi}(s, z, a)\big[1 -\hat{P}_z^a(\sigma^a_s)\big]&\\
&\hspace{.1in} = \sum_{s, z \ne s, a} \sum_{t=2}^{T-1} \gamma^t H^{t}_{\xi,
  \pi}(s, z, a)\big[1 -\hat{P}_z^a(\sigma^a_s)\big]&\\
&\hspace{.1in}= \sum_{s, z \ne s, a} \sum_{t=1}^{T-2} \gamma^{t+1} H^{t+1}_{\xi,
  \pi}(s, z, a)\big[1 -\hat{P}_z^a(\sigma^a_s)\big]&\\
&\hspace{.1in} \le \sum_{s, z \ne s, a} \sum_{t=1}^{T-2} \gamma^{t+1} H^{t+1}_{\xi,
  \pi}(s, z, a)&\\
&\hspace{.1in} = \gamma \sum_{s', z', a'}\bigg(\sum_{s, z\ne s} P_{s'}^{a'}(s)\hat{P}_{z'}^{a'}(z)\bigg)\sum_{t=1}^{T-2}\gamma^t H^t_{\xi,
  \pi}(s', z', a')&\\
&\hspace{.1in} = \gamma \sum_{s', z', a'}\big(1-
\hat{P}_{z'}^{a'}(\sigma_{s'}^{a'})\big)\sum_{t=1}^{T-2}\gamma^t H^t_{\xi,
  \pi}(s', z', a')&\\
&\hspace{.1in} = \gamma \sum_{t=1}^{T-2} \gamma^t \expect_{(s, z, a) \sim H^{t}_{\xi,
  \pi}}\big[1 -\hat{P}_z^a(\sigma^a_s)\big].&
\end{align*}

Combining the two reveals a recurrence relation.
\begin{align*}
& \sum_{t = 1}^{T-1} \gamma^t\expect_{(s, z, a) \sim H^{t}_{\xi, \pi}}\big[1 -
  \hat{P}_z^a(\sigma^a_s)\big]&\\
&\hspace{.2in} \le \sum_{t = 1}^{T-1}\gamma^t
  \expect_{(s, a) \sim D^{t}_{\xi, \pi}} \big[1 -
  \hat{P}_s^a(\sigma^a_s)\big]&\\
&\hspace{1in} + \gamma \sum_{t=1}^{T-2} \gamma^t \expect_{(s, z, a) \sim H^{t}_{\xi,
  \pi}}\big[1 -\hat{P}_z^a(\sigma^a_s)\big].&
\end{align*}

Unrolling the recurrence, we see that
\begin{align*}
& \sum_{t = 1}^{T-1} \gamma^t\expect_{(s, z, a) \sim H^{t}_{\xi, \pi}}\big[1 -
  \hat{P}_z^a(\sigma^a_s)\big]&\\
&\hspace{.75in} \le \sum_{j=1}^{T-1} \gamma^{j-1}\sum_{t = 1}^{T-j}\gamma^t
  \expect_{(s, a) \sim D^{t}_{\xi, \pi}} \big[1 - \hat{P}_s^a(\sigma^a_s)\big].&
\end{align*}

Now note that
\begin{align*}
\sum_{j=1}^{T-1}& \gamma^{j-1}\sum_{t = 1}^{T-j}\gamma^t
  \expect_{(s, a) \sim D^{t}_{\xi, \pi}} \big[1 - \hat{P}_s^a(\sigma^a_s)\big]&\\
&= \sum_{t = 1}^{T-1}\sum_{j=1}^{T-t} \gamma^{j-1}\gamma^t
  \expect_{(s, a) \sim D^{t}_{\xi, \pi}} \big[1 - \hat{P}_s^a(\sigma^a_s)\big]&\\
&= \sum_{t = 1}^{T-1}\gamma^t
  \expect_{(s, a) \sim D^{t}_{\xi, \pi}} \big[1 - \hat{P}_s^a(\sigma^a_s)\big]\sum_{j=1}^{T-t} \gamma^{j-1}&\\
&= \sum_{t = 1}^{T-1}\gamma^t
  \expect_{(s, a) \sim D^{t}_{\xi, \pi}} \big[1 - \hat{P}_s^a(\sigma^a_s)\big]\frac{1 - \gamma^{T-t}}{1 - \gamma}&\\
&= \frac{1}{1-\gamma}\sum_{t = 1}^{T-1}(\gamma^t - \gamma^{T})
  \expect_{(s, a) \sim D^{t}_{\xi, \pi}} \big[1 - \hat{P}_s^a(\sigma^a_s))\big].&\qedhere
\end{align*}
\end{proof}

\begin{lem} 
  In H-DAgger-MC, the policies $\hat{\pi}_{1:N}$ are such that for any
  policy $\pi$,
\begin{align*}
  \expect_{s
  \sim \mu}\big[V^{\pi}(s) - V^{\bar{\pi}}(s)\big] \le \frac{8
  M}{1-\gamma} c_{\nu}^{\pi} \sum_{t = 1}^{T-1}
  \bar{\epsilon}^t_{prd} + \bar{\epsilon}_{mc}.
\end{align*}
\end{lem}
\begin{proof}
Recall that 
\begin{align*}
&\expect_{s \sim \mu}\big[V^{\pi}(s) - V^{\bar{\pi}}(s)\big] = \frac{1}{N}\sum_{n = 1}^N \expect_{s \sim \mu} \big[V^{\pi}(s) -
V^{\hat{\pi}_n}(s)\big].&
\end{align*}
and by Lemma 1 for any $n \ge 1$, 
\begin{align*}
&\expect_{s \sim \mu} \big[V^{\pi}(s) -
V^{\hat{\pi}_n}(s)\big] \le&\\
&\hspace{.45in}\frac{4}{1 - \gamma} \expect_{(s, a)
  \sim \xi^{\pi, \hat{\pi}_n}_{\mu}}[|\hat{Q}_{T,n}^{\rho}(s, a) - Q_T^{\rho}(s, a)|] +\bar{\epsilon}_{mc},&
\end{align*}
where 
\begin{align*}
\xi^{\pi, \hat{\pi}_n}_{\mu}(s, a) = \frac{1}{2} D_{\mu, \hat{\pi}_n}&(s, a) + \frac{1}{4}
  D_{\mu, \pi}(s, a)\\ 
+ \frac{1}{4}\Big(&(1 - \gamma) \mu(s)
    \hat{\pi}_n(a \mid s)\\  
&+\gamma \sum_{z, b} D_{\mu, \pi}(z, b)
    P_{z}^{b}(s) \hat{\pi}_n(a \mid s) \Big).
\end{align*} 
Combining this with Lemma \ref{lem:detblind},
\begin{align*}
&\frac{1}{N} \sum_{n=1}^N \frac{4}{1 - \gamma} \expect_{(s, a)
  \sim \xi^{\pi, \hat{\pi}_n}_{\mu}}[|\hat{Q}_{T,n}^{\rho}(s, a) - Q_T^{\rho}(s, a)|] +\bar{\epsilon}_{mc}&\\
& \hspace{.1in} \le \frac{1}{N} \sum_{n=1}^N \frac{8M}{1 - \gamma} \sum_{t =
  1}^{T-1} \gamma^t&\\
&\hspace{.8in} \expect_{(s, z, a) \sim H^{t,n}_{\xi^{\pi,
      \hat{\pi}_n}_{\mu}, \rho}}\big[1 - \hat{P}_n^t(\sigma_s^a \mid z, a)\big] +\bar{\epsilon}_{mc}&\\
& \hspace{.1in} \le \frac{8M}{1 - \gamma} \frac{1}{N} \sum_{n=1}^N \sum_{t =
  1}^{T-1}&\\
&\hspace{.8in} \expect_{(s, z, a) \sim H^{t, n}_{\xi^{\pi,
      \hat{\pi}_n}_{\mu}, \rho}}\big[1 - \hat{P}_n^t(\sigma_s^a \mid z, a)\big] +\bar{\epsilon}_{mc}.&
\end{align*}
\vfill \break
Now note that for any $t$ and any $n$,
\begin{align*}
&\expect_{(s, z, a) \sim H^{t, n}_{\xi^{\pi,
      \hat{\pi}_n}_{\mu}, \rho}}\big[1 - \hat{P}_n^t(\sigma_s^a \mid z, a)\big]&\\
&= \frac{1}{2} \sum_{s', a'} D_{\mu, \hat{\pi}_n}(s', a') \expect_{(s, z, a) \sim H^{t, n}_{s', a', \rho}}\big[1 -
\hat{P}_n^t(\sigma_s^a\mid z, a)\big]&\\
&\hspace{.15in}+\frac{1}{4} \sum_{s', a'} D_{\mu, \pi}(s', a') \expect_{(s, z, a) \sim H^{t, n}_{s', a', \rho}}\big[1 -
\hat{P}_n^t(\sigma_s^a \mid z, a)\big]&\\
&\hspace{.15in}+\frac{\gamma}{4} \sum_{s', a'} \sum_{s'', a''}D_{\mu, \pi}(s'',
a'')P_{s''}^{a''}(s')\hat{\pi}_n(a' \mid s')&\\
&\hspace{1.3in} \expect_{(s, z, a) \sim H^{t, n}_{s', a', \rho}}\big[1 -
\hat{P}_n^t(\sigma_s^a \mid z, a)\big]&\\
&\hspace{.15in}+\frac{1- \gamma}{4} \sum_{s', a'}
\mu(s')\hat{\pi}_n(a' \mid s')&\\
&\hspace{1.3in} \expect_{(s, z, a) \sim H^{t, n}_{s', a', \rho}}\big[1 -
\hat{P}_n^t(\sigma_s^a \mid z, a)\big]&\\
&\le \frac{1}{2} \sum_{s', a'} D_{\mu, \hat{\pi}_n}(s', a') \expect_{(s, z, a) \sim H^{t, n}_{s', a', \rho}}\big[1 -
\hat{P}_n^t(\sigma_s^a \mid z, a)\big]&\\
&\hspace{.15in}+\frac{1}{4} c^\pi_\nu \sum_{s', a'} \nu(s', a') \expect_{(s, z, a) \sim H^{t, n}_{s', a', \rho}}\big[1 -
\hat{P}_n^t(\sigma_s^a \mid z, a)\big]&\\
&\hspace{.15in}+\frac{\gamma}{4} c^\pi_\nu \sum_{s', a'} \sum_{s'', a''}\nu(s'',
a'')P_{s''}^{a''}(s')\hat{\pi}_n(a' \mid s')&\\
&\hspace{1.3in} \expect_{(s, z, a) \sim H^{t, n}_{s', a', \rho}}\big[1 -
\hat{P}_n^t(\sigma_s^a \mid z, a)\big]&\\
&\hspace{.15in}+\frac{1- \gamma}{4} \sum_{s', a'}
\mu(s')\hat{\pi}_n(a' \mid s')&\\
&\hspace{1.3in} \expect_{(s, z, a) \sim H^{t, n}_{s', a', \rho}}\big[1 -
\hat{P}_n^t(\sigma_s^a \mid z, a)\big]&\\
&\le c^\pi_\nu \bigg(\frac{1}{2} \sum_{s', a'} D_{\mu,
  \hat{\pi}_n}(s', a')&\\
&\hspace{1.3in} \expect_{(s, z, a) \sim H^{t, n}_{s', a', \rho}}\big[1 -
\hat{P}_n^t(\sigma_s^a \mid z, a)\big]&\\
&\hspace{.15in}+\frac{1}{4} \sum_{s', a'} \nu(s', a') \expect_{(s, z, a) \sim H^{t, n}_{s', a', \rho}}\big[1 -
\hat{P}_n^t(\sigma_s^a \mid z, a)\big]&\\
&\hspace{.15in}+\frac{\gamma}{4} \sum_{s', a'} \sum_{s'', a''}\nu(s'',
a'')P_{s''}^{a''}(s')\hat{\pi}_n(a' \mid s')&\\
&\hspace{1.3in} \expect_{(s, z, a) \sim H^{t, n}_{s', a', \rho}}\big[1 -
\hat{P}_n^t(\sigma_s^a \mid z, a)\big]&\\
&\hspace{.15in}+\frac{1- \gamma}{4} \sum_{s', a'} \mu(s')\hat{\pi}_n(a'
\mid s')&\\
&\hspace{1.3in} \expect_{(s, z, a) \sim H^{t, n}_{s', a', \rho}}\big[1 -
\hat{P}_n^t(\sigma_s^a \mid z, a)\big]\bigg)&\\
&=c^\pi_\nu \expect_{(s, z, a) \sim H^{t, n}_{\xi_n, \rho}}\big[1 - \hat{P}_n^t(\sigma_s^a \mid z, a)\big].&
\end{align*}
\clearpage
When $t = 1$,
$\expect_{(s, z, a) \sim H^{t, n}_{\xi_n, \rho}}\big[1 -
\hat{P}_n^t(\sigma_s^a \mid z, a)\big] = \expect_{(s, a) \sim \xi_n(s,
  a)}\big[1 - \hat{P}_n^t(\sigma_s^a \mid s, a)\big]$. When $t > 1$,
\begin{align*}
&\expect_{(s, z, a) \sim H^{t, n}_{\xi_n, \rho}}\big[1 -
\hat{P}_n^t(\sigma_s^a \mid z, a)\big]&\\
&= \sum_{s_t, z_t, a_t}\expect_{(s_1, a_1) \sim
  \xi_n}\bigg[\sum_{a_{1:t-1}}\rho(a_{2:t} \mid a_1)&\\
&\hspace{.5in}P_{s_1}^{a_{0:t-1}}(s_t \mid
s_1, a_{0:t-1})\hat{P}_{n}^{1:t-1}(z_t \mid
s_1,a_{0:t-1})\bigg]&\\
&\hspace{2in}\big(1 - \hat{P}_{n}^{t}(\sigma_{s_t}^{a_t} \mid z_t, a_t)\big)&\\
&=\expect_{(s, z, a) \sim H^t_n}\big[1 -
\hat{P}_n^t(\sigma_s^a \mid z, a)\big].&
\end{align*}

Thus, putting it all together, we have shown that
\begin{align*}
  &\expect_{s \sim \mu}\big[V^{\pi}(s) - V^{\bar{\pi}}(s)\big] &\\
&\hspace{.1in}\le
  \frac{8M}{1 - \gamma} c^\pi_{\nu}\frac{1}{N} \sum_{n=1}^N
  \sum_{t = 1}^{T-1}\expect_{(s, z, a) \sim H^t_n}\big[1 -
\hat{P}_n^t(\sigma_s^a \mid z, a)\big]&\\
&\hspace{2.9in} + \bar{\epsilon}_{mc}&\\
&\hspace{.1in}= \frac{8M}{1 - \gamma} c^\pi_{\nu}\sum_{t =
  1}^{T-1}\bar{\epsilon}^t_{prd} + \bar{\epsilon}_{mc}.&\qedhere
\end{align*}
\end{proof}

\end{document}